\documentclass{article}
\usepackage[nonatbib,preprint]{neurips_2023}

\usepackage{macros}

\title{Bandits with Replenishable Knapsacks:\\ the Best of both Worlds}

\author{
    Martino Bernasconi\\
    Politecnico di Milano\\
    \texttt{martino.bernasconideluca@polimi.it}
    \And 
    Matteo Castiglioni\\
    Politecnico di Milano\\
    \texttt{matteo.castiglioni@polimi.it}\\
    \And 
    Andrea Celli\\
    Bocconi University \\
    \texttt{andrea.celli2@unibocconi.it}\\
    \And 
    Federico Fusco\\
    Sapienza Università di Roma\\
    \texttt{federico.fusco@uniroma1.it}\\
}

\begin{document}

\maketitle

\begin{abstract}
The bandits with knapsack (BwK) framework models online decision-making problems in which an agent makes a sequence of decisions subject to resource consumption constraints. The traditional model assumes that each action consumes a non-negative amount of resources and the process ends when the initial budgets are fully depleted. We study a natural generalization of the BwK framework which allows non-monotonic resource utilization, i.e., resources can be replenished by a positive amount. We propose a best-of-both-worlds primal-dual template that can handle any online learning problem with replenishment for which a suitable primal regret minimizer exists. In particular, we provide the first positive results for the case of adversarial inputs by showing that our framework guarantees a constant competitive ratio $\alpha$ when $B=\Omega(T)$ or when the possible per-round replenishment is a positive constant. Moreover, under a stochastic input model, our algorithm yields an instance-independent $\tilde \cO(T^{1/2})$ regret bound which complements existing instance-dependent bounds for the same setting. Finally, we provide applications of our framework to some economic problems of practical relevance.
\end{abstract}

\section{Introduction}

We study online learning problems in which a decision maker tries to maximize their cumulative reward over a time horizon $T$, subject to a set of $m$ resource-consumption constraints. At each $t$, the decision maker plays an action $x_t\in\cX$, and subsequently observes a realized reward $f_t(x_t)$, with $f_t:\cX\to[0,1]$, and an $m$-dimensional vector of resource consumption $\vc_t(x_t)$.

Our framework extends the well-known \emph{Bandits with Knapsacks} (BwK) framework of \citet{Badanidiyuru2018jacm}. In the BwK model, the resource consumption is \emph{monotonic} (\ie $\vc_t(\cdot)\in[0,1]^m$ for all $t\in [T]$).
This framework has numerous motivating applications ranging from dynamic pricing to online ad allocation (see, \eg \cite{besbes2009dynamic,babaioff2012dynamic,wang2014close,badanidiyuru2012learning,combes2015bandits}), and it has been extended in numerous directions such as modeling adversarial inputs \cite{immorlica2022jacm} and other non-stationary input models \cite{celli2023best,fikioris2023approximately}, more general notions of resources and constraints \cite{agrawal2019bandits}, contextual and combinatorial bandits \cite{badanidiyuru2014resourceful,agrawal2016efficient,sankararaman2018combinatorial}.   

\citet{kumar2022non} recently proposed a natural generalization of the BwK model in which resource consumption can be \emph{non-monotonic}, that is, resources can be replenished or renewed over time so costs are no longer required to be such that $\vc_t(\cdot)\succeq\vec{0}$. We call such model \emph{Bandits with Replenishable Knapsacks} (BwRK).

\xhdr{Contributions.} We propose a general primal-dual template that can handle online learning problems in which the decision maker has to guarantee some long-term resource-consumption constraints and resources can be renewed over time.
We show that our framework provides \emph{best-of-both-worlds} guarantees in the spirit of \citet{balseiro2022best}: it guarantees a regret bound of $\tilde O(T^{1/2})$ in the case in which $(f_t,\vc_t)$ are i.i.d. samples from a fixed but unknown distribution, and it guarantees a constant-factor competitive ratio in the case in which budgets grow at least linearly in $T$, or when the possible per-round replenishment is a positive constant. We remark that known best-of-both-worlds frameworks like the one by \citet{balseiro2022best} cannot be applied to this setting as they assume monotonic resource consumption. In that case, we show that our framework recovers the state-of-the-art rate of $1/\rho$ by \citet{castiglioni2022online}, where $\rho$ is the per-iteration budget. 
Our primal-dual template is applicable to any online problem for which a suitable primal regret minimizer is available. Therefore, we first provide general guarantees for the framework without making any assumption on the primal and dual regret minimizers being employed (\Cref{sec:guarantees meta}). Then, we show how such regret minimizers should be chosen depending on the information available to the learner about the intensity of the budget replenishment (\Cref{sec:choosing}). Moreover, we provide explicit bounds that only depend on the guarantees of the primal regret minimizer.
In particular, we show how the primal and dual minimizers should be instantiated in the case in which the amount of resources that can be replenished at each time $t$ is known, and in the more challenging case in which it is unknown a-priori to the decision maker. Finally, we demonstrate the flexibility of our framework by instantiating it in some relevant settings (\Cref{sec:applications}). First, we instantiate the framework in the BwRK model by \citet{kumar2022non}, thereby providing the first positive results for BwRK under adversarial inputs, and the first instance-independent regret bound for the stochastic setting. The latter complements the instance-dependent analysis by \citet{kumar2022non}.
Then, we apply the framework to a simple inventory management problem, and to revenue maximization in bilateral trade.

\section{Preliminaries}\label{sec: preliminaries}

Vectors are denoted by bold fonts. Given vector $\vx$, let $\vx[i]$ be its $i$-th component. The set $\{1,\ldots,n\}$, with $n\in\N_{>0}$, is denoted as $[n]$. Finally, given a discrete set $S$, we denote by $\Delta^{S}$ the $|S|$-simplex.

\subsection{Basic set-up}
There are $T$ rounds and $m$ resources. The decision maker has an arbitrary non-empty set of available strategies $\cX$.
In each round $t\in[T]$, the decision maker chooses $x_t\in\cX$, and subsequently observes a reward function $f_t:\cX\to[0,1]$, and a function $\vc_t:\cX\to[-1,1]^{m}$ specifying the consumption or replenishment of each of the $m$ resources. 
Each resource $i\in[m]$ is endowed with an initial budget of $B$ to be spent over the $T$ steps. \footnote{For ease of notation we consider a uniform intial buget, but the case of different initial budgets easily follows from our results.} 
We denote by $\rho$ the per-iteration budget, which is such that $B=T\rho$, and we let $\vrho\defeq \rho\vone\in\R_{>0}^m$.
For $i\in[m]$ and $x\in\cX$, if $c_{t,i}(x)<0$ we say that at time $t$ action $x$ \emph{restores} a positive amount to the budget available for the $i$-th resource. If $c_{t,i}(x)>0$, we say that action $x$ at time $t$ \emph{depletes} some of the available budget for the $i$-th resource.

Let $\gamma_t\defeq(f_t,c_t)$ be the input pair at time $t$, and $\vgamma_{T}\defeq(\gamma_1,\gamma_2,\ldots,\gamma_T)$ be the sequence of inputs up to time $T$. 
The repeated decision making process stops at the end of the time horizon $T$. The goal of the decision maker is to maximize their cumulative reward $\sum_{t=1}^T f_t(x_t)$ while satisfying the resource-consumption constraints $\sum_{t=1}^T c_{t,i}(x_t)\le T\cdot\rho$ for each $i\in [m]$.   

Given two functions $f:\cX\to\R$ and $\vc:\cX\to\R^m$, we denote by $\cL_{f,\vc}:\cX \times \R^m_{\ge 0}\to \R$ the Lagrangian function defined as 
\[ \cL_{f,\vc}(x,\vlambda)\defeq f(x)+ \langle \vlambda, \vrho-\vc(x)\rangle\,\,\textnormal{ for all }x\in\cX,\vlambda\in\R^{m}_{\ge 0}.\]
Given $\cX$, the set of \emph{strategy mixtures} $\Xi$ is the set of probability measures on the Borel sets of $\cX$. Moreover, we endow $\cX$ with the Lebesgue $\sigma$-algebra and assume that all possible functions $f_t,\vc_t$ are measurable with respect to every probability measure $\vxi \in \Xi$.

\subsection{Baseline adversarial setting} 
Given a sequence of inputs $\vgamma_T$ selected by an oblivious adversary, the baseline for the  adversarial setting is $\OPT_{\vgamma}:=\sup_{x\in\cX}\sum_{t=1}^Tf_t(x)$, which is the total expected reward of the best \emph{fixed unconstrained} strategy in hindsight belonging to $\cX$. 
Given $t\in[T]$, we write $\OPT_{\vgamma,t}$ to denote the expected reward of the best fixed feasible policy for the sequence of inputs restricted to $(\gamma_1,\ldots,\gamma_t)$.
Moreover, for any sequence of inputs $\vgamma_T$, and $t\in[T]$, let $\tilde f_\tau: \cX\to[0,1]$ and $\tilde \vc_t:\cX\to [0,1]^m$ be such that:
\begin{equation}\label{eq: tilde functions}
\tilde f_t(x)\defeq \frac{1}{t}\sum_{s=1}^t f_t(x)\hspace{.2cm}\textnormal{and}\hspace{.2cm}\tilde \vc_\tau(x)\defeq \frac{1}{t}\sum_{s=1}^t \vc_s(x),\,\quad \forall x\in\cX.
\end{equation}
Then, for $t\in[T]$, we define $\OPT_{\tilde f_t}\defeq \sup_{x \in \cX}\tilde f_t(x)$, and the baseline for the adversarial setting can be rewritten as $\OPT_{\vgamma}=T\cdot  \OPT_{\tilde f_T}$.
In the setting with monotonic resource utilization and adversarial inputs, previous work usually employs weaker baselines (see, \eg \cite{immorlica2022jacm,castiglioni2022online,castiglioni2022unifying}). For example, \citet{castiglioni2022online} considers the reward attained by the best fixed strategy mixture until budget depletion, after which the void action is played. We show that, despite the stronger baseline, we match the state-of-the-art $1/\rho$ competitive-ratio by \citet{castiglioni2022online} when $\beta=0$.
We will work under the following standard assumption. 

\begin{assumption}\label{ass:adv}
    There exists a \emph{void action} $\nullx\in \cX$ and a constant $\beta\ge0$ such that $c_{t,i}(\nullx)\le-\beta$, for all resources $i \in [m]$ and $t\in [T]$.
\end{assumption}
Notice that when $\beta=0$ we recover the standard assumption of BwK (see, \eg\cite{Badanidiyuru2018jacm,immorlica2022jacm}).
We will often parametrize regret bounds using $\nu \defeq \beta+\rho$. %
Intuitively, this parameter measures how much budget is available at each iteration, and how fast the available budget can be replenished.

\subsection{Baseline stochastic setting} 

In the stochastic version of the problem, each input $\gamma_t=(f_t,\vc_t)$ is drawn i.i.d. from some fixed but unknown distribution $\distr$ over a set of possible input pairs. 
Let $\bar f:\cX\to[0,1]$ be the expected reward function, and $\bar \vc:\cX\to[0,1]^m$ be the expected resource-consumption function (where both expectations are taken with respect to $\distr$).

Given two arbitrary measurable functions $f:\cX\to[0,1]$, $\vc:\cX\to[-1,1]^m$, we define the following linear program, which chooses the strategy mixture $\xi$  that maximizes the reward $f$, while keeping the expected consumption of every resource $i\in[m]$ given $c$ below a target $\rho$:
\begin{equation}\label{eq:opt lp gen}
\OPTLP_{f,c}\defeq\mleft\{\hspace{-1.25mm}\begin{array}{l}
\displaystyle
\sup_{\xi\in\Xi}\E_{x\sim\xi}\mleft[f(x)\mright] \\ [2mm]
\displaystyle \text{\normalfont s.t. } \E_{x\sim\xi}\mleft[ \vc(x)\mright]\preceq\vrho
\end{array}\mright.,
\end{equation}
In the stochastic setting, our baseline is $\OPTLP_{\bar f,\bar \vc}$. It is well-known that $T\cdot \OPTLP_{\bar f,\bar \vc}$ is an upper bound on the expected reward of any algorithm (see, \eg, \citep[Lemma 3.1]{Badanidiyuru2018jacm} and \citep[Lemma 2.1]{kumar2022non}).

\begin{lemma}[Lemma 2.1 of \citet{kumar2022non}]\label{lemma: stoc opt ub}
In the stochastic setting, the total expected reward of any algorithm is at most $T \cdot \OPTLP_{\bar f, \bar \vc}$.
\end{lemma}

In the stochastic setting we make the same ``positive drift'' assumption of \citet{kumar2022non}, which is weaker than our Assumption~\ref{ass:adv} in the adversarial case.
\begin{assumption}\label{ass:stoc}
    There exists of a \emph{void action} $\nullx\in \cX$ such that, for all resources $i\in [m]$, it holds $\E\mleft[c_{i}(\nullx)\mright]\le-\beta$, where $\beta\ge0$ and the expectation is with respect to the draw of the $\vc$ from $\distr$.
\end{assumption}

When $\beta=0$, we recover the standard assumption for the monotonic case by  \citet{badanidiyuru2013bandits}.

\subsection{Regret minimization}\label{sec:reg min}

We will consider regret minimizers for a set $\cW$ as generic algorithms that implements two functions: (i) the function $\nextelement$ returns an element $w_t\in\cW$, and (ii) the function $\observeutility[\cdot]$ which takes some feedback and updates the internal state of the regret minimizer. 
In the \emph{full-feedback} model the regret minimizer observes as feedback a function $u_t:\cX\to\mathbb{R}$, while in the \emph{bandit-feedback} model it observes only the realized $u_t(w_t)$.
The standard objective of a regret minimizer is to control the cumulative regret with respect to a set $\cY \subseteq \cW$ defined as $\cR_T(\cY):=\sup_{w\in\cY}\sum_{t=1}^T(u_t(w)-u_t(w_t))$. 
In the following, we will also exploit a more general notion of regret, in which the regret minimizer suffers regret only in specific rounds. In particular, given 
a subset of rounds $\cT\subset [T]$ we define $\cR_\cT(\cY):=\sup_{w\in \cY}\sum_{t\in\cT}(u_t(w)-u_t(w_t))$.
Then, we can recover common notions of regret such as standard (external) regret for which $\cT=[T]$, and \emph{weakly-adaptive regret} for which $\ireg_T(\cY):=\sup_{\cI=[t_1,t_2]\subseteq [T]} \cR_{\cI}(\cY)$ \cite{hazan2007adaptive}. We remove the dependency from $\cY$ when $\cY=\cW$ (\eg we write $\cR_T$ in place of $\cR_T(\cW)$).

\section{Primal-dual template}

We assume to have access to two regret minimizers with the following characteristic. 
A bandit-feedback \emph{primal} regret minimizer $\cRp$ which outputs a strategy $x_t\in\cX$ at each $t$, and subsequently receives as feedback the realized utility function $\lossp[t](x_t)= f_t(x_t)+\langle\vlambda_t, \vrho-\vc_{t}(x_t)\rangle$, and a full-feedback \emph{dual} regret minimizer $\cRd$ that receives as input the utility function: $\lossd:\vlambda\mapsto\langle\vlambda, \vc_t(x_t)-\vrho\rangle$. note that the dual regret minimizer always has full feedback by construction.\footnote{We focus on the more challenging bandit-feedback setting but our results easily extend to the full-feedback setting.}

\cref{alg:meta alg} summarizes the structure of our primal-dual template. For each $t$, if the available budget $B_{t,i}$ is less than 1 for some resource $i$, the algorithm plays the void action $\nullx$ and updates the budget accordingly. This ensures that the budget will never fall below 0. Otherwise, the regret minimizer plays action $x_t$ at time $t$, which is determined by invoking \nextelement. Then, $\lossp(x_t)$ and $\lossd$ are observed, and the budget consumption is updated according to the realized costs $\vc_t$. 
If the budget was at least $1$, the internal state of the two regret minimizers is updated via \observeutility[\cdot], on the basis of the feedback specified by the primal loss $\lossp(x_t)$, and the dual loss function $\lossd[t]$.
The algorithm terminates when the time horizon $T$ is reached.

\begin{algorithm}[tb]
	\caption{Primal-Dual template}
	\label{alg:meta alg}
	\begin{algorithmic}[1]
		\State {\bfseries Input:} parameters $B,T$; regret minimizers $\cRp$ and $\cRd$
		\State {\bfseries Initialization:} $ B_{1,i}\gets B,\forall i\in[m]$; initialize $\cRp,\cRd$; $\cTG=\{\emptyset\}$, $\cT_\nullx=\{\emptyset\}$.
		\For{$t = 1, 2, \ldots , T$}
		\If{$ \exists\, i \in [m]:B_{t,i} < 1$}
            \State $\cT_\nullx \gets \cT_\nullx \cup \{t\}$
		\State {\bfseries Primal action:} $x_t\gets\nullx$
		\State {\bfseries Observe costs:} Observe $\vc_t(\nullx)$ and update available resources: $\vB_{t+1}\gets \vB_{t} - \vc_t(\nullx)$
		\Else
            \State $\cTG \gets \cTG \cup \{t\}$
		\State {\bfseries Dual decision:} $\vlambda_t \gets\cRd.\nextelement$
		\State{\bfseries Primal decision:} $x_t \gets \cRp.\nextelement$		
		\State {\bfseries Observe cost:} Observe $\vc_t(x_t)$ and update available resources: $\vB_{t+1}\gets \vB_{t} - \vc_t(x_t)$
		\State {\bfseries Primal update:} 
		\begin{itemize}[leftmargin=2cm]
			\item $\lossp(x_t)\gets  f_t(x_t)  +   \langle \vlambda_t, \vrho-\vc_t(x_t)\rangle$ 
			\item  $\cRp.\observeutility[\lossp(x_t)]$ 
		\end{itemize}

		\State {\bfseries Dual update:} 
		\begin{itemize}[leftmargin=2cm]
			\item $\lossd:\mathbb{R}^d \ni \vlambda \mapsto  \langle\vlambda, \vc_t(x_t)-\vrho\rangle$
			\item $\cRd.\observeutility[\lossd]$
		\end{itemize}
	\EndIf
		\EndFor
	\end{algorithmic}
\end{algorithm}

We partition the set of rounds in two disjoint sets $\cTG \subseteq [T]$ and $\cT_\nullx \subseteq [T]$. The set
$\cTG\defeq\{t\in[T]:\forall i\in[m], B_{t,i}\ge1\} $ includes all the rounds in which all the resources were at least $1$, and hence the regret minimizers $\cRp$ and $\cRd$ were actually invoked.
On the other hand, $\cT_\nullx\defeq\{t\in[T]:\exists i\in[m], B_{t,i}<1\}$ is the set of rounds in which at least one resource is smaller than $1$. Clearly, we have $\cTG \cup \cT_\nullx= [T]$.
Then, let $\tau\in \cT_\nullx$ be the last time in which the budget was strictly less then $1$ for at least one resource, \ie $\tau=\max \cT_\nullx$. 
We partition $\cTG$ in two sets $\cTG[<\tau]$ and $\cTG[>\tau]$ which are the rounds in $\cTG$ before and after $\tau$, respectively.
Formally, $\cTG[<\tau] \defeq [\tau]\setminus \cT_\nullx$, and $\cTG[>\tau]\defeq [T]\setminus [\tau]$.

We denote by $\cump$ (resp., $\cumd$) the cumulative regret incurred by $\cRp$ (resp., $\cRd$). Following the notation introduced in \Cref{sec:reg min} we will write $\cump[\cTG]$ to denote the regret accumulated by the primal regret minizer over time steps in $\cTG$. 
 Let $\mathcal{D}\defeq \{\vlambda \in \mathbb{R}_+^d:\lVert\vlambda\rVert_1\le 1/\nu\} $.
We consider consider the regret of the dual in the set of rounds $\cTG[<\tau]$ and $\cTG[>\tau]$, with respect to the action set $\cD$.
Formally, $\cumd[\cTG[<\tau]](\cD)= \sup_{\vlambda \in \cD} \sum_{t \in \cTG[<\tau]} \lossd(\vlambda) - \lossd(\vlambda_t) $. The term $\cumd[\cTG[>\tau]]$ is defined analogously. 
Finally, let $M\defeq \max_{t}\|\vlambda\|_1$ be the largest value of the $\ell_1$-norm of dual multipliers over the time horizon.

\section{General guarantees of the primal-dual template}\label{sec:guarantees meta}

In this section, we provide no-regret guarantees of the general template described in \Cref{alg:meta alg}.

\subsection{Adversarial setting}\label{sec:adv}

We start by describing the guarantees of \Cref{alg:meta alg} in the adversarial setting.
The idea is that, since $\tau$ corresponds to the last time in which at least one resource had $B_{t,i}<1$, it must be the case that the primal ``spent a lot'' during the time intervals before $\tau$ in which it played (i.e., $\cTG[<\tau]$). Ideally, the dual regret minimizer should adapt to this behavior and play a large $\vlambda$ in $\cTG[<\tau]$, thereby attaining a large cumulative utility in $\cTG[<\tau]$. We show that the Lagrange multipliers in $\cD$ are enough for this purpose.
Then, as soon as we reach $t>\tau$, the dual should adapt and start to play a small $\vlambda$. Indeed, during these rounds the primal regret minimizer gains resources and therefore the optimal dual strategy would be setting $\vlambda=\mathbf{0}$.
The effectiveness of the dual regret minimizer in understanding in which phase it is playing, and in adapting to it by setting high/small penalties, is measured by the size of the regret terms $\cumd[\cTG[<\tau]](\cD)$ and $\cumd[\cTG[>\tau]](\cD)$. In particular, $\cumd[\cTG[<\tau]](\cD)$ (resp., $\cumd[\cTG[>\tau]](\cD)$) is low if the dual regret minimizer behaves as expected before $\tau$ (resp., after $\tau)$.
Intuitively,  if $\cRd$ guarantees that those terms are small, then the dual regret minimizer is able to react quickly to the change in the behavior of the primal player before and after $\tau$. As a byproduct of this, we show that the part of the primal's cumulative utility due to the Lagrangian penalties is sufficiently small. At the same time, we know that the primal regret minimizer has regret at most $\cump[\cTG]$ with respect to a strategy mixture that plays the optimal fixed unconstrained strategy with probability $\nu/(1+\beta)$, and the void action otherwise. This strategy mixture guarantees a $\nu/(1+\beta)$ fraction of the optimal utility without violating the constraints at any rounds. Formally, we can show the following.\footnote{All omitted proofs can be found in the appendix.}

\begin{restatable}{theorem}{theoremadv}\label{thm:adv}
Let $\alpha\defeq \nu/(1+\beta)$. In the adversarial setting, Algorithm~\ref{alg:meta alg} 
outputs a sequence of actions $(x_t)_{t=1}^T$ such that
\[
\sum_{t \in [T]} f_t(x_t)\ge  \alpha\cdot \OPT_{\vgamma}-\left(\frac{2}{\nu}+\cumd[\cTG[<\tau]](\cD) +\cumd[\cTG[>\tau]](\cD)+\cump[\cTG]\right).
\]
\end{restatable}
Notice that, in the case of $\beta=0$, we recover the standard guarantees of adversarial bandits with knapsacks for the case in which $B=\Omega(T)$~\cite{castiglioni2022online}, where the competitive ratio $\alpha$ is exactly $\rho$. As expected, the possibility of replenishing resources yields an improved competitive ratio. 

\xhdr{Remark.} In order for \Cref{thm:adv} to provide a meaningful bound, we need the three regret terms on the right-hand side to be suitably upperbounded by some term sublinear in $T$ (see \Cref{sec:choosing}). Since the time steps in $\cTG$ are the only rounds in which $\cRp$ is invoked, any standard regret minimizer can be used to bound $\cump[\cTG]$. However, the same does not holds for $\cRd$. Indeed, we need a regret minimizer which can provide suitable regret upper bounds to $\cumd[\cTG[<\tau]](\cD)$ and $\cumd[\cTG[>\tau]](\cD)$, at the same time. One cannot simply bound $\cumd[\cTG[<\tau]](\cD)+\cumd[\cTG[>\tau]](\cD)$ by $\cumd[\cTG](\cD)$, since the best action in the sets $\cTG[<\tau]$ and $\cTG[>\tau]$ may differ. Standard regret minimizers usually do not provide this guarantee, so we need special care in choosing $\cRd$. In particular, we need a weakly adaptive dual regret minimizer.

\subsection{Stochastic setting}\label{sec:stoc}

In this setting, we can exploit stochasticity of the environment to show that the expected utility of the primal under the sequence of realized inputs $\vgamma=(f_t,\vc_t)_{t=1}^T$ is close to the primal expected utility at $(\bar f,\bar\vc)$, and to provide a suitable upperbound to the amount by which each resource is replenished during $\cT_\nullx$. Given $\delta\in(0,1]$, let $\cume:=\sqrt{8T\log\left(4mT/\delta\right)}$.

\begin{restatable}{lemma}{lemmastoc}\label{lm:HoefEmpty}
    For any $\xi\in\Xi$ and $\delta\in(0,1]$, with probability at least $1-\delta$, it holds that:
    \begin{align}
        & \textstyle{\sum_{t \in \cT_\nullx} c_{t,i}(\nullx) \le -\beta |\cT_\nullx|+ M\cume,\, \forall i \in [m],}\quad\textnormal{and}\label{eq:HoefEmpty2} \\
        & \textstyle{\mathop\mathbb{E}_{x\sim\xi}\left[\sum_{t \in  \cTG } f_{t}(x)+\langle\vlambda_t,\vrho-\vc_t(x)\rangle \right]\ge \mathop{\mathbb{E}}\limits_{x\sim\xi}\left[\sum_{ t \in  \cTG} \bar f(x)+\langle\vlambda_t,\vrho-\bar \vc(x)\rangle\right]-  M\cume.}\label{eq:HoefEmpty3}
    \end{align}
    
\end{restatable}

Then, we can prove the following regret bound. 

\begin{restatable}{theorem}{thstoc}\label{th:stoc}
    Let the inputs $(f_t, c_t)$ be i.i.d. samples from a fixed but unknown distribution $\distr$. For $\delta\in(0,1]$, we have that with probability at least $1-\delta$, it holds
    \[
    \sum_{t =1}^T f_t(x_t)\ge T\cdot \OPTLP_{\bar f,\bar c}-\left(\frac{2}{\nu}+\frac{1}{\nu}\cume+\cumd[\cTG[< \tau]](\cD)+\cumd[\cTG[>\tau]](\cD)+\cump[\cTG]\right).
    \]
\end{restatable}
The proof follows a similar approach to the one of \Cref{thm:adv}, with two main differences. First, we can now exploit standard concentration inequalities to relate realizations of random variables with their mean. Second, we exploit the fact that the primal has regret at most $\cump[\cTG]$ against an optimal solution to $\OPTLP_{\bar f,\bar\vc}$, which is feasible in expectation and has an expected utility that matches the value of our baseline. This allows us to obtain competitive-ratio equal to 1 in the stochastic setting.

\section{Choosing appropriate regret minimizers}\label{sec:choosing}

In order to have meaningful guarantees in both the adversarial and stochastic setting, we need to choose the regret minimizers $\cRp$ and $\cRd$ so that $\cump[\cTG]$, $\cumd[\cTG[<\tau]](\cD)$ and $\cumd[\cTG[>\tau]](\cD)$ all grow sublinearly in $T$.
In the following section we will discuss two different scenarios, which differ in the amount of information which the decision maker is required to have. In \Cref{sec:knownbeta} we are going to assume that the decision maker knows the per-round replenishment factor $\beta$. Then, in Section~\ref{sec:interval1}, we will show that this assumption can be removed by employing a primal regret minimizer $\cRp$ with slightly stronger regret guarantees.
In both cases, the dual regret minimizer $\cRd$ has to be weakly adaptive, since both terms $\cumd[\cTG[<\tau]](\cD)$ and $\cumd[\cTG[>\tau]](\cD)$ need to be sublinear in $T$.
On the other hand, we will make minimal assumptions on the primal regret minimizer $\cRp$. Its choice largely depends on the application being considered, as we show in \Cref{sec:applications}. In general, the primal regret minimizer must meet the minimal requirement of guaranteeing a sublinear regret upper bound $\uppp[T,\delta]$ with probability at least $1-\delta$, when the adversarial rewards are in $[0,1]$.

\subsection{Implementing \Cref{alg:meta alg} with known replenishment factor}\label{sec:knownbeta}

We start by assuming that the decision maker knows $\beta$ or, more generally, a lower bound $\tilde\beta$ on it. Therefore, the algorithm can compute $\nu$, or its lower bound. 
When $\tilde\beta$ is known, we can instantiate the regret minimizer $\cRd$ to play on the set $\tilde\cD:=\{\vlambda\in\mathbb{R}_+^d:\|\vlambda\|_1\le 1/\tilde \nu\}\supseteq \cD $,  where $\tilde \nu:=\tilde\beta+\rho\le\nu$.

As briefly discussed above, the dual regret minimizer on the set of Lagrange multipliers $\cD$ must be weakly adaptive.
This can be achieved via variations of the \emph{fixed share algorithm} originally proposed by \citet{herbster1998tracking}. We will exploit the \emph{generalized share algorithm} \cite{bousquet2002tracking} with the analysis of \citet{Cesa2012}, which is applicable since $\cRd$ has full feedback by construction.

The set $\tilde\cD$ can be written as $\tilde\cD = \textnormal{co}\left\{\vzero, \nicefrac{1}{\tilde\nu}\vone_i \textnormal{ with } i\in[m]\right\}$,
where $\vone_i$ is the vector with the $i$-th component equal to $1$ and equal to $0$ otherwise, and $\co$ denotes the convex hull.
Since $\cD\subseteq \tilde \cD$, achieving no-weakly-adaptive regret with respect to $\tilde \cD$ implies the same result for $\cD$.
Thus, we can instantiate the share algorithm on the $(m+1)$-simplex, and since the losses observed by $\cRp$ are linear it is possible to prove that: 
\begin{lemma}[{\cite[Corollary~2]{Cesa2012}}]
For any $0<\tilde\beta\le\beta$, there exists an algorithm that guarantees:
\[
\max\left(\cumd[\cTG[<\tau]](\cD), \cumd[\cTG[>\tau]](\cD)\right)\le\frac{2}{\tilde\nu}\sqrt{T\log\left( 2mT\right)}.
\]
\end{lemma}
Since the dual regret minimizer $\cRd$ can play any Lagrange multiplier $\vlambda_t\in\tilde\cD$, we have that the rewards observed by the primal regret minimizer $\cRp$ are in the range $[0,1+2/\tilde\nu]$ because it holds
\begin{align*}
\sup\limits_{x\in \cX, t\in [T]}|\lossp[t](x)|&\le \sup\limits_{x\in \cX, t\in [T]}\Big\{|f_t(x)|+\lVert\vlambda_t\rVert_1\cdot\lVert\vrho-\vc_t(x)\rVert_\infty\Big\}\le1+\frac{2}{\tilde\nu} \le \frac{4}{\tilde\nu}.
\end{align*}
With probability at least $\delta$, the primal regret minimizer $\cRp$ guarantees a regret $\uppp[T,\delta]$ against rewards in $[0,1]$. Then, by re-scaling the realized rewards before giving them in input to the regret minimizer, we get a regret bound of $\frac{4}{\tilde\nu}\uppp$ against rewards $\lossp[t](\cdot)$ that are in $[0,\nicefrac{4}{\tilde\nu}]$. This simple construction is applicable because the range of the rewards is known.
By combining these observations we can easily recover the following corollary of Theorem~\ref{thm:adv} and Theorem~\ref{th:stoc}.

\begin{corollary}\label{cor:known beta}
    Assume that the dual regret minimizer is generalized fixed share on $\tilde \cD$, and that the primal regret minimizer has regret at most $\uppp[\cTG,\delta]$ against losses in $[0,1]$ with probability at least $1-\delta$, for $\delta\in(0,1]$. 
    In the adversarial setting, for any $\tilde\beta\le\beta$, with probability at least $1-\delta$ \Cref{alg:meta alg} guarantees that
    \[
    \sum_{t \in [T]} f_t(x_t)\ge \alpha OPT_{\vgamma} - \left(\frac{2}{\nu}+\frac{1}{\tilde\nu}\sqrt{T\log\left( 2mT\right)}+\frac{4}{\tilde\nu}\uppp[\cTG,\delta]\right),
    \]
    where $\alpha=\nicefrac{\nu}{1+\beta}$. In the stochastic setting, with probability at least $1-2\delta$, \Cref{alg:meta alg} guarantees 
    \[
    \sum_{t \in [T]} f_t(x_t)\ge T\cdot \OPTLP_{\bar f,\bar c}-\left(\frac{2}{\nu}+\frac{1}{\nu}{\cume}+\frac{1}{\tilde\nu}\sqrt{T\log\left( 2mT\right)}+\frac{4}{\tilde\nu}\uppp[\cTG,\delta]\right).
    \]
\end{corollary}

\subsection{Implementing \Cref{alg:meta alg} with unknown replenishment factor}\label{sec:interval1}

In this section we will show how to implement Algorithm~\ref{alg:meta alg} when no information about the per-round replenishment factor $\beta$ is available.
This impacts both the primal and the dual regret minimizer.
Differently from the previous section, we cannot instantiate the regret minimizer $\cRd$ directly on $\cD$ (or on a larger set $\tilde\cD$) since we do not know such set. Therefore, we need a dual regret minimizer that plays on $\mathbb{R}_+^m$, but has sublinear weakly-adaptive regret with respect to Lagrange multipliers in $\cD$.
To achieve this, we can use the Online Gradient Descent (OGD) algorithm, instantiated on $\mathbb{R}_+^m$ with starting point $\vlambda_0=\mathbf{0}$. Indeed, it is well known that OGD guarantees that the regret on any interval of rounds $[t_1,t_2]\subseteq [T]$ is upper bounded by the $\ell_2$ distance between the point played at $t_1$ and the comparator. In our setting this is equivalent to the following lemma (see \citet[Chapter 10]{hazan2016introduction}).
\begin{lemma}\label{lem:OGD}
    For any $\cTG\subset[T]$ and any $t_1,t_2\in \cTG$, if the dual regret minimizer is OGD with learning rate $\eta$, we have that the regret with respect to $\vlambda$ is upper bounded by
    \[
    \cumd[\cTG\cap[t_1,\ldots ,t_2]](\{\vlambda\})\le \frac{\|\vlambda-\vlambda_{t_1}\|^2_2}{2\eta}+\frac{1}{2} \eta\, m\, T.
    \]
\end{lemma}

In this setting, we also need additional assumptions on the primal regret minimizer $\cRp$. Formally we need an algorithm that satisfies the following condition. 
\begin{assumption}\label{ass:primal_no_beta}
    We assume that, for any $\cT=[t_1, t_2]\subseteq[T]$, the weakly-adaptive regret of $\cRd$ facing adversarial losses with unknown range $L$ is upper bounded by $L^2\uppp[T,\delta]$ with probability at least $1-\delta$, where $\uppp[T,\delta]$ is independent from the range of payoffs.
\end{assumption}

In order to provide the final regret bound for this setting, we need to show that the size of Lagrange multipliers remains bounded by a suitable term.

\begin{restatable}{lemma}{lemmaboundLM}\label{lem:bounded_lagrangian}
    Assume that the dual regret minimizer is OGD on $\mathbb{R}_+^{m}$ with $\eta=(k_1 \cume +k_2 m \uppp[T,\delta]+2m\sqrt{T})^{-1}$, where $k_1, k_2$ are absolute constants, and
    the primal regret minimizer $\cRp$ satisfies Assumption~\ref{ass:primal_no_beta}. Then, both in the adversarial and stochastic setting, the Lagrange multipliers $\vlambda_t$ played by the dual regret minimizer $\cRd$ are such that $M:=\sup_{t\in [T]}\|\vlambda_t\|_1\le 8m/\nu$.
\end{restatable}

This result extends the similar result of~\citet[Theorem~6.2]{castiglioni2023online} to the case of multiple constraints.
Lemma~\ref{lem:bounded_lagrangian} allows us bound the regret $\cump[\cTG]$ of the primal, and the regret terms $\cumd[\cTG<\tau](\cD)$ and $\cumd[\cTG>\tau](\cD)$ of the dual regret minimizer. 
In particular, for what concerns the dual regret minimizer $\cRd$, we can bound the maximum distance between Lagrange multipliers which are played by $\cRd$. Then, we can bound the regret terms $\cumd[\cTG<\tau](\cD)$ and $\cumd[\cTG>\tau](\cD)$ via \Cref{lem:OGD}. Similarly, we can bound the regret of the primal through Assumption~\ref{ass:primal_no_beta}.
Using this observations, together with \Cref{thm:adv} and \Cref{th:stoc}, we can extend \Cref{cor:known beta} to the case of unknown $\beta$. 

\begin{restatable}{corollary}{corollaryFINAL}\label{cor:unknown beta}
    Assume that the dual regret minimizer is OGD on $\mathbb{R}_+^{m}$ with $\eta=(k_1 \cume +k_2 m \uppp[T,\delta]+2m\sqrt{T})^{-1}$, where $k_1, k_2$ are absolute constants, and
    the primal regret minimizer $\cRp$ satisfies Assumption~\ref{ass:primal_no_beta}. Then, in the adversarial setting, Algorithm~\ref{alg:meta alg} guarantees with probability at least $1-\delta$ that
    \[
    \sum_{t \in [T]} f_t(x_t)\ge \alpha\cdot\OPT_{\vgamma} - k_3\frac{m^4}{\nu^2}\left(\uppp[T,\delta]+\cume\right),
    \]
    where $k_3$ is an absolute constant and $\alpha=\nicefrac{\nu}{1+\beta}$. In the stochastic setting, it guarantees with probability at least $1-2\delta$ that
    \[
    \sum_{t \in [T]} f_t(x_t)\ge T\cdot \OPTLP_{\bar f,\bar c}-k_4\frac{m^4}{\nu^2}\left(\uppp[T,\delta]+\cume\right),
    \]
    where $k_4$ is an absolute constant.
\end{restatable}

\section{Applications}\label{sec:applications}

    This section demonstrates the flexibility of our framework by studying three well motivated models: BwRK, inventory management, and revenue maximization in bilateral trade.

    \subsection{Bandits with replenishable knapsacks (BwRK)}\label{sec:bwrk}

        Consider the standard BwK problem: at each time step $t$, the learner selects an action $i_t$ out of $K$ actions (thus $\cX=[K]$), suffers a loss $\ell_t(i_t)\in [0,1]$ and incurs a cost vector $\vc_t(i_t)$ that specifies the consumption of each one of its $m$ resources. In BwRK, cost vectors may have negative components: $\vc_t(\cdot) \in [-1,1]^m$.
        We focus on the most challenging scenario, in which the parameter $\beta$ is not known. We instantiate the primal-dual framework (\Cref{alg:meta alg}) using  EXP3-SIX~\cite{neu2015explore} as the primal regret minimizer, while online gradient descent is employed as the dual regret minimizer. \citet{castiglioni2023online} show that EXP3-SIX achieve weakly-adaptive regret $L^2 \tilde \cO(\sqrt{KT})$, where $L$ is the range of the observed losses.

        First, let's consider the adversarial case, where losses and cost functions are generated by an oblivious adversary. \Cref{ass:adv} is verified when there exists a null action $\nullx$ that always yields non-negative resource replenishment (\ie there exists  $\beta\ge 0$ s.t. $c_{t,j}(\nullx) \le -\beta$ for each resource $j$ and time $t$).

        \begin{theorem}\label{thm:adv_bandits}
            Consider the BwRK problem in the adversarial setting. There exists an algorithm satisfying the following bound on the regret: 
            \[
            \alpha\cdot \OPT_{\vgamma}-\sum_{t \in [T]} f_t(x_t)\le \tilde \cO\left(\frac{m^4}{\nu^2}\sqrt{KT} \log\mleft(\frac{1}{\delta}\mright)\right)  ,
            \]
            with probability at least $1-\delta$, where $\alpha:=\nicefrac{\nu}{1+\beta}$.
        \end{theorem}

     This is the first positive result for the BwRK problem in the adversarial setting. 
     
     Second, we consider the stochastic version of the problem, in which  losses and cost vectors are drawn i.i.d.~from an unknown distribution. In this setting, the void action assumption (\Cref{ass:stoc}) requires the existence of a distribution $\nullx \in \Delta_K$ over the actions such that, in expectation over the draws of the cost function, it verifies $\E[c_j(\nullx)]\le -\beta$ for all resources $j$\footnote{Note, the role of $\beta$ is played in \cite{kumar2022non} by the parameter $\delta_{\text{slack}}$. Furthermore, there the authors make additional assumptions on the structure of the optimal LP and its solution.}.  The analysis of our primal-dual framework allows us to show that the same learning algorithm presented for the adversarial setting yields the following instance-independent results in the stochastic setting.
        
        \begin{theorem}\label{thm:bwrkstoc}
            Consider the BwRK in the stochastic setting. There exists an algorithm satisfying, with probability at least $1-2\delta$, the following bound on the regret:
            \[
            T\cdot \OPTLP_{\bar f,\bar c}-\sum_{t =1}^T f_t(x_t)\le \tilde \cO\left(\frac{m^4}{\nu^2}\sqrt{KT} \log\mleft(\frac{1}{\delta}\mright)\right). 
            \]
        \end{theorem}

        \Cref{thm:bwrkstoc} provides the first instance-independent regret bound under i.i.d. inputs, and it complements the instance-dependent analysis by \citet{kumar2022non}. In particular, in the stochastic setting, they provide a logarithmic instance-dependent bound on the regret of order $O(\nicefrac{Km^2}{\Delta^2}\log T)$, where $\Delta$ is a notion of suboptimality gap (already present in \cite{LiSY21}), which in principle may be arbitrarily small. Although our work does not offer instance-dependent logarithmic bounds on the regret, we advance the study of the BwRK problem along three main directions: i) we design an algorithm which handles at the same time both stochastic and adversarial inputs (while \citet{kumar2022non} only deal with the stochastic input model), ii) our primal-dual approach is arguably simpler, and iii)  we provide the first worst-case dependence on the time horizon $T$.

        \subsection{Economic applications and discussion}
        Even tough the application of our primal-dual template to bandit with knapsacks allows us to extend numerous standard BwK applications to the case with replenishment, it still leaves out some specific scenarios. More specifically, dealing with \emph{negative} rewards introduces additional challenges. We show how this challenge can be easily circumvented in the stochastic case, while we leave as an open problem the adversarial case. 
        In the following, we describe two of such applications showing how to apply our framework to the stochastic setting. Then, we conclude the section by describing the challenges in extending the results to the adversarial setting.

        \xhdr{Inventory management.} As a simple yet meaningful application of our techniques, consider the following inventory management problem. Each day $t$, a shopkeeper is confronted with a decision: either open for business and attempt to sell the goods it has in store, or travel to its supplier in order to restore its inventory. This simplified model captures many real world application that contemplates the trade-off between exploiting the available inventory, and ``skipping a turn'' to replenish it. For the sake of clarity we restrict ourselves to the case of a single resource and supplier, however this could be easily extended to more general instances.
        Formally, the goods in stock at the beginning of each day are $B_t \ge 0$, and the shopkeeper has a set two actions $\cX=\{o,s\}$. Action $o$ corresponds to opening for business with reward $r_t(o) \in [0,1]$ and inventory consumption $c_t(o) \in [0,1]$, while action $s$ corresponds to going to the supplier, with $r_t(s)\in [-1,0]$ and negative resource consumption $c_t(s) \in [-1,0]$ which both depend on the supplier's availability and current price. 
        Clearly, it is possible to select action $o$ only if $B_t \ge 1$, while action $s$ plays the role of the void action $\nullx$. We focus of the stochastic case in which $\E[c_t(s)]\le -\beta$. In particular, \Cref{ass:stoc} is naturally satisfied. Specifically, $\beta$ represents the expected amount of good available from the supplier. %
        Notice that, for the stochastic case, it is sufficient to translate and re-scale the rewards back to be in $[0,1]$. By employing EXP-SIX as in \Cref{sec:bwrk}, this immediately yields a $\tilde {\cO}(\sqrt{T})$ regret bound via \Cref{cor:unknown beta}.

        \xhdr{Bilateral trade.} Many applications involve making buying and selling decisions which happen simultaneously ``on the same day''.
        This is captured by the well known bilateral trade model \citep{vickrey1961counterspeculation,myerson1983efficient,cesa2023bilateral}: each day, the learner (\ie the \emph{merchant}) posts two prices, a price $p_t \in \mathbb{R}_+$ and a price $q_t \in \mathbb{R}_+$ to a seller and a buyer of a good, respectively. The seller (resp., the buyer) has a private valuation of $s_t \in [0,1]$ (resp., $b_t \in (0,1]$). If the valuation $s_t$ is smaller than $p_t$, then the merchant buys one unit of good from the seller, while if the valuation $b_t$ is larger than the price $q_t$ (and there is some inventory left) then the merchant sells a unit of the good to the buyer. The merchant can only sell a good if its stock is such that $B_t\ge1$, \ie the inventory contains at least one unit of it. Moreover, we assume that the merchant has no initial budget, \ie $B_1=0$. The merchant's utility (also called \emph{revenue} in the bilateral trade literature), at time $t$ is defined by
        \[
            \term{rev}_t(p_t,q_t) \defeq q_t \indicator{q_t \le b_t} - p_t \indicator{s_t \le p_t}.
        \]
        Similarly, for any $t\in [T]$, the stock consumption is updated as follows:
        \[
            B_{t+1} = B_{t} - \indicator{q_t \le b_t} + \indicator{s_t \le p_t}.
        \]
        Therefore, the strategy of choosing $p_t=1$ and $q_t>1$ surely increases the budget by one unit, and it has revenue $\term{rev}_t(p_t,q_t)=-1$. This strategy plays the role of the void action $\nullx$, with per-round replenishment parameter $\beta=1$. Similarly to the example above, in the stochastic setting it is enough to translate and re-scale the revenue to satisfy all the assumptions of our algorithm. This immediately yields the  $\tilde\cO({\sqrt{T}})$ regret guarantees.

        \xhdr{Challenges of the adversarial setting.}
        In the case in which rewards are allowed to be negative and inputs are adversarial, the simple simple trick of suitably re-scaling and translating the utilities is not applicable. This is due to the fact that our algorithm guarantees a multiplicative approximation of the benchmark and this multiplicative factor is clearly \emph{not} invariant with respect to translations. We leave as an open problem the question of how to handle application scenarios in which rewards may be negative in the adversarial setting.  Addressing this challenge would likely necessitate the development of novel techniques to resolve the issue. We observe that the results for the adversarial setting would be preserved if we could provide a void action $\nullx$ that has non-negative reward and negative cost, but this is not usually the case in practical applications. Finally, it would be interesting to extend the simplified inventory-management model which we discussed to more complex models which have been proposed in the operations research literature (see, \eg \cite{chen2022dynamic,chen2004coordinating}).

\clearpage

\section*{Acknowledgements}
Matteo Castiglioni is supported by the PNRR-PE-AI FAIR project funded by the NextGeneration EU program. The work of Federico Fusco is partially supported by ERC Advanced Grant 788893 AMDROMA “Algorithmic and Mechanism Design Research in Online Markets”, PNRR MUR project PE0000013-FAIR, and PNRR MUR project IR0000013-SoBigData.it.

\bibliographystyle{plainnat}
\bibliography{bib-abbrv,refs}

\begin{thebibliography}{31}
\providecommand{\natexlab}[1]{#1}
\providecommand{\url}[1]{\texttt{#1}}
\expandafter\ifx\csname urlstyle\endcsname\relax
  \providecommand{\doi}[1]{doi: #1}\else
  \providecommand{\doi}{doi: \begingroup \urlstyle{rm}\Url}\fi

\bibitem[Agrawal and Devanur(2019)]{agrawal2019bandits}
Shipra Agrawal and Nikhil~R Devanur.
\newblock Bandits with global convex constraints and objective.
\newblock \emph{Operations Research}, 67\penalty0 (5):\penalty0 1486--1502,
  2019.

\bibitem[Agrawal et~al.(2016)Agrawal, Devanur, and Li]{agrawal2016efficient}
Shipra Agrawal, Nikhil~R Devanur, and Lihong Li.
\newblock An efficient algorithm for contextual bandits with knapsacks, and an
  extension to concave objectives.
\newblock In \emph{29th Annual Conference on Learning Theory (COLT)}, 2016.

\bibitem[Babaioff et~al.(2012)Babaioff, Dughmi, Kleinberg, and
  Slivkins]{babaioff2012dynamic}
Moshe Babaioff, Shaddin Dughmi, Robert Kleinberg, and Aleksandrs Slivkins.
\newblock Dynamic pricing with limited supply.
\newblock In \emph{Proceedings of the 13th ACM Conference on Electronic
  Commerce}, pages 74--91, 2012.

\bibitem[Badanidiyuru et~al.(2012)Badanidiyuru, Kleinberg, and
  Singer]{badanidiyuru2012learning}
Ashwinkumar Badanidiyuru, Robert Kleinberg, and Yaron Singer.
\newblock Learning on a budget: posted price mechanisms for online procurement.
\newblock In \emph{Proceedings of the 13th ACM conference on electronic
  commerce}, pages 128--145, 2012.

\bibitem[Badanidiyuru et~al.(2013)Badanidiyuru, Kleinberg, and
  Slivkins]{badanidiyuru2013bandits}
Ashwinkumar Badanidiyuru, Robert Kleinberg, and Aleksandrs Slivkins.
\newblock Bandits with knapsacks.
\newblock In \emph{2013 IEEE 54th Annual Symposium on Foundations of Computer
  Science, FOCS 2013}, pages 207--216. IEEE, 2013.

\bibitem[Badanidiyuru et~al.(2014)Badanidiyuru, Langford, and
  Slivkins]{badanidiyuru2014resourceful}
Ashwinkumar Badanidiyuru, John Langford, and Aleksandrs Slivkins.
\newblock Resourceful contextual bandits.
\newblock In \emph{Conference on Learning Theory}, pages 1109--1134. PMLR,
  2014.

\bibitem[Badanidiyuru et~al.(2018)Badanidiyuru, Kleinberg, and
  Slivkins]{Badanidiyuru2018jacm}
Ashwinkumar Badanidiyuru, Robert Kleinberg, and Aleksandrs Slivkins.
\newblock Bandits with knapsacks.
\newblock \emph{J. ACM}, 65\penalty0 (3), 2018.

\bibitem[Balseiro et~al.(2022)Balseiro, Lu, and Mirrokni]{balseiro2022best}
Santiago~R Balseiro, Haihao Lu, and Vahab Mirrokni.
\newblock The best of many worlds: Dual mirror descent for online allocation
  problems.
\newblock \emph{Operations Research}, 2022.

\bibitem[Besbes and Zeevi(2009)]{besbes2009dynamic}
Omar Besbes and Assaf Zeevi.
\newblock Dynamic pricing without knowing the demand function: Risk bounds and
  near-optimal algorithms.
\newblock \emph{Operations Research}, 57\penalty0 (6):\penalty0 1407--1420,
  2009.

\bibitem[Bousquet and Warmuth(2002)]{bousquet2002tracking}
Olivier Bousquet and Manfred~K Warmuth.
\newblock Tracking a small set of experts by mixing past posteriors.
\newblock \emph{Journal of Machine Learning Research}, 3\penalty0
  (Nov):\penalty0 363--396, 2002.

\bibitem[Castiglioni et~al.(2022{\natexlab{a}})Castiglioni, Celli, and
  Kroer]{castiglioni2022online}
Matteo Castiglioni, Andrea Celli, and Christian Kroer.
\newblock Online learning with knapsacks: the best of both worlds.
\newblock In \emph{International Conference on Machine Learning}, pages
  2767--2783. PMLR, 2022{\natexlab{a}}.

\bibitem[Castiglioni et~al.(2022{\natexlab{b}})Castiglioni, Celli, Marchesi,
  Romano, and Gatti]{castiglioni2022unifying}
Matteo Castiglioni, Andrea Celli, Alberto Marchesi, Giulia Romano, and Nicola
  Gatti.
\newblock A unifying framework for online optimization with long-term
  constraints.
\newblock In \emph{Advances in Neural Information Processing Systems},
  volume~35, pages 33589--33602, 2022{\natexlab{b}}.

\bibitem[Castiglioni et~al.(2023)Castiglioni, Celli, and
  Kroer]{castiglioni2023online}
Matteo Castiglioni, Andrea Celli, and Christian Kroer.
\newblock Online learning under budget and {ROI} constraints and applications
  to bidding in non-truthful auctions.
\newblock \emph{arXiv preprint arXiv:2302.01203}, 2023.

\bibitem[Celli et~al.(2023)Celli, Castiglioni, and Kroer]{celli2023best}
Andrea Celli, Matteo Castiglioni, and Christian Kroer.
\newblock Best of many worlds guarantees for online learning with knapsacks.
\newblock \emph{arXiv preprint arXiv:2202.13710}, 2023.

\bibitem[Cesa-Bianchi et~al.(2012)Cesa-Bianchi, Gaillard, Lugosi, and
  Stoltz]{Cesa2012}
Nicol\`{o} Cesa-Bianchi, Pierre Gaillard, G\'{a}bor Lugosi, and Gilles Stoltz.
\newblock Mirror descent meets fixed share (and feels no regret).
\newblock In \emph{Proceedings of the 25th International Conference on Neural
  Information Processing Systems - Volume 1}, page 980–988, 2012.

\bibitem[Cesa-Bianchi et~al.(2023)Cesa-Bianchi, Cesari, Colomboni, Fusco, and
  Leonardi]{cesa2023bilateral}
Nicol{\`o} Cesa-Bianchi, Tommaso Cesari, Roberto Colomboni, Federico Fusco, and
  Stefano Leonardi.
\newblock Bilateral trade: A regret minimization perspective.
\newblock \emph{Mathematics of Operations Research}, 2023.

\bibitem[Chen et~al.(2022)Chen, Simchi-Levi, Wang, and Zhou]{chen2022dynamic}
Boxiao Chen, David Simchi-Levi, Yining Wang, and Yuan Zhou.
\newblock Dynamic pricing and inventory control with fixed ordering cost and
  incomplete demand information.
\newblock \emph{Management Science}, 68\penalty0 (8):\penalty0 5684--5703,
  2022.

\bibitem[Chen and Simchi-Levi(2004)]{chen2004coordinating}
Xin Chen and David Simchi-Levi.
\newblock Coordinating inventory control and pricing strategies with random
  demand and fixed ordering cost: The finite horizon case.
\newblock \emph{Operations research}, 52\penalty0 (6):\penalty0 887--896, 2004.

\bibitem[Combes et~al.(2015)Combes, Jiang, and Srikant]{combes2015bandits}
Richard Combes, Chong Jiang, and Rayadurgam Srikant.
\newblock Bandits with budgets: Regret lower bounds and optimal algorithms.
\newblock \emph{ACM SIGMETRICS Performance Evaluation Review}, 43\penalty0
  (1):\penalty0 245--257, 2015.

\bibitem[Fikioris and Tardos(2023)]{fikioris2023approximately}
Giannis Fikioris and {\'E}va Tardos.
\newblock Approximately stationary bandits with knapsacks.
\newblock \emph{arXiv preprint arXiv:2302.14686}, 2023.

\bibitem[Hazan and Seshadhri(2007)]{hazan2007adaptive}
Elad Hazan and Comandur Seshadhri.
\newblock Adaptive algorithms for online decision problems.
\newblock In \emph{Electronic colloquium on computational complexity (ECCC)},
  volume~14, 2007.

\bibitem[Hazan et~al.(2016)]{hazan2016introduction}
Elad Hazan et~al.
\newblock \emph{Introduction to online convex optimization}, volume~2.
\newblock Now Publishers, Inc., 2016.

\bibitem[Herbster and Warmuth(1998)]{herbster1998tracking}
Mark Herbster and Manfred~K Warmuth.
\newblock Tracking the best expert.
\newblock \emph{Machine learning}, 32\penalty0 (2):\penalty0 151--178, 1998.

\bibitem[Immorlica et~al.(2022)Immorlica, Sankararaman, Schapire, and
  Slivkins]{immorlica2022jacm}
Nicole Immorlica, Karthik Sankararaman, Robert Schapire, and Aleksandrs
  Slivkins.
\newblock Adversarial bandits with knapsacks.
\newblock \emph{J. ACM}, 69\penalty0 (6), 2022.
\newblock ISSN 0004-5411.

\bibitem[Kumar and Kleinberg(2022)]{kumar2022non}
Raunak Kumar and Robert Kleinberg.
\newblock Non-monotonic resource utilization in the bandits with knapsacks
  problem.
\newblock In \emph{Advances in Neural Information Processing Systems
  (NeurIPS)}, 2022.

\bibitem[Li et~al.(2021)Li, Sun, and Ye]{LiSY21}
Xiaocheng Li, Chunlin Sun, and Yinyu Ye.
\newblock The symmetry between arms and knapsacks: {A} primal-dual approach for
  bandits with knapsacks.
\newblock In \emph{{ICML}}, volume 139 of \emph{Proceedings of Machine Learning
  Research}, pages 6483--6492. {PMLR}, 2021.

\bibitem[Myerson and Satterthwaite(1983)]{myerson1983efficient}
Roger~B Myerson and Mark~A Satterthwaite.
\newblock Efficient mechanisms for bilateral trading.
\newblock \emph{Journal of economic theory}, 29\penalty0 (2):\penalty0
  265--281, 1983.

\bibitem[Neu(2015)]{neu2015explore}
Gergely Neu.
\newblock Explore no more: Improved high-probability regret bounds for
  non-stochastic bandits.
\newblock \emph{Advances in Neural Information Processing Systems}, 28, 2015.

\bibitem[Sankararaman and Slivkins(2018)]{sankararaman2018combinatorial}
Karthik~Abinav Sankararaman and Aleksandrs Slivkins.
\newblock Combinatorial semi-bandits with knapsacks.
\newblock In \emph{International Conference on Artificial Intelligence and
  Statistics}, pages 1760--1770. PMLR, 2018.

\bibitem[Vickrey(1961)]{vickrey1961counterspeculation}
William Vickrey.
\newblock Counterspeculation, auctions, and competitive sealed tenders.
\newblock \emph{The Journal of finance}, 16\penalty0 (1):\penalty0 8--37, 1961.

\bibitem[Wang et~al.(2014)Wang, Deng, and Ye]{wang2014close}
Zizhuo Wang, Shiming Deng, and Yinyu Ye.
\newblock Close the gaps: A learning-while-doing algorithm for single-product
  revenue management problems.
\newblock \emph{Operations Research}, 62\penalty0 (2):\penalty0 318--331, 2014.

\end{thebibliography}
\clearpage
\appendix
\section{Proofs Omitted from Section~\ref{sec:adv}}
\theoremadv*

\begin{proof}

    We divided the proof in two steps. First, we show that over the rounds $\cTG$ in which the regret minimizers play the ``lagrangified'' cumulative costs is controlled. Then, in the second step, we show that over the same set of rounds $\cTG$ the ``lagrangified'' utility is large. In order to simplify the notation, we will write $\cumd[\cTG[<\tau]]$ (resp., $\cumd[\cTG[>\tau]]$) in place of $\cumd[\cTG[<\tau]](\cD)$ (resp., $\cumd[\cTG[>\tau]](\cD)$).
    
    \xhdr{Part 1.}
    First, we show that in $\cTG[<\tau]$ the budget spent by the decision maker is sufficiently high. Let $i^*$ be the resource that had the lowest budget in round $\tau\in \cT_\nullx$.

    Summing the budget equation $B_{t+1,i^*}=B_{t,i^\ast}-c_{t,i^*}(x_t)$ over $t=\{1,\ldots,\tau-1\}$, gives us
    \[
    B_{\tau,i^*}-B_{i^*} = -\sum\limits_{t\le \tau-1} c_{t,i^*}(x_t).
    \]
    Since by definition $B_{\tau,i^*}<1$ and $B_{1,i^*}=B$ we get that:
    \begin{equation}\label{eq:th11}
        \sum\limits_{t\le\tau-1}c_{t,i^*}(x_t)> B-1.
    \end{equation}
    Moreover, in rounds $t\in \cT_\nullx$, since the budget was strictly less we have by construction that  $x_t=\nullx$, and thus $c_{t,i^*}(x_t)\le -\beta$. This readily implies that:
    \begin{equation}\label{eq:th12}
        \sum\limits_{t\in\cT_\nullx}c_{t,i^*}(x_t)\le -\beta|\cT_\nullx|.
    \end{equation}
    By combining Equation~\eqref{eq:th11} and Equation~\eqref{eq:th12} we obtain that:
    \begin{align}
        \sum\limits_{t\in \cTG[<\tau]} c_{t,i^*}(x_t) &= \sum\limits_{t\in [\tau]} c_{t,i^*}(x_t) - \sum\limits_{t\in \cT_\nullx} c_{t,i^*}(x_t)\\
        &\ge B-2+\beta |\cT_\nullx|.
    \end{align}
    Since the budget spent in rounds $\cTG[<\tau]$ is large, it must be the case that the dual regret minimizer $\cRd$ collects a large cumulative utility in those rounds. Formally, the regret $\cumd[\cTG[<\tau]]$ of $\cRd$ on rounds $\cTG[<\tau]$ with respect to $\vlambda = \frac{1}{\nu} \boldsymbol{1}_{i^*}$ reads as follows
    \begin{align}
        \sum\limits_{t\in\cTG[<\tau]}\langle\vlambda_t, \vc_t(x_t)-\vrho\rangle&\ge \frac{1}{\nu}\sum\limits_{t\in\cTG[<\tau]} (c_{t,i^*}(x_t)-\vrho_{i^*})-\cumd[\cTG[<\tau]]\nonumber\\
        &\ge \frac{1}{\nu}(B-2+\beta |\cT_\nullx|-|\cTG[<\tau]|\rho_{i^*})-\cumd[\cTG[<\tau]]\nonumber\\
        &=\frac{1}{\nu}(\rho_{i^*}T-2+\beta |\cT_\nullx|-|\cTG[<\tau]|\rho_{i^*})-\cumd[\cTG[<\tau]]\nonumber\\
        &=\frac{1}{\nu}(\rho_{i^*}(|\cTG[<\tau]|+|\cT_\nullx|+|\cTG[>\tau]|)-2+\beta |\cT_\nullx|-|\cTG[<\tau]|\rho_{i^*})-\cumd[\cTG[<\tau]]\nonumber\\
        &= \frac{1}{\nu}(\nu|\cT_\nullx|+\rho_{i^*}|\cTG[>\tau]|-2)-\cumd[\cTG[<\tau]]\nonumber\\
        &\ge \frac{1}{\nu}(\nu|\cT_\nullx|-2)-\cumd[\cTG[<\tau]]\nonumber\\
        &=|\cT_\nullx|-\frac{2}{\nu}-\cumd[\cTG[<\tau]]\label{eq:th14}.
    \end{align}

    Then, by considering $\vlambda=\boldsymbol{0}$ and the definition of the regret on $\cTG[>\tau]$ we get that the utility of dual is bounded by:
    \begin{equation}\label{eq:th15}
        \sum\limits_{t\in \cTG[>\tau]} \langle\vlambda_t, \vc_t(x_t)-\vrho\rangle\ge -\cumd[\cTG[>\tau]].
    \end{equation}

    Thus, by combining the inequalities of Equation~\eqref{eq:th14} and Equation~\eqref{eq:th15} we can conclude that:

    \begin{equation}\label{eq:th15_2}
        \sum\limits_{t\in\cTG}\langle\vlambda_t, \vc_t(x_t)-\vrho\rangle\ge|\cT_\nullx|-\frac{2}{\nu}-\cumd[\cTG[<\tau]] -\cumd[\cTG[>\tau]].
    \end{equation}
    
    \xhdr{Part 2.}
    Now, let $x^*$ be the best unconstrained strategy, \ie $x^*\in\arg\max_{x\in\cX} \sum_{t\in [T]} f_t(x)$ and thus $\sum_{t\le T} f_t(x^*):=\OPT_{\vgamma}$.\footnote{For simplicity we replaced the supremum with the maximum. Results would continue to hold with $\sup$ with minor modifications.} Consider the mixed  strategy $\xi^*\in\Xi$ that randomizes between $x^*$ and $\nullx$, with probability $\frac{\rho+\beta}{1+\beta}$ and $\frac{1-\rho}{1+\beta}$, respetively.
    Then $\mathop{\mathbb{E}}_{x\sim\xi^*}\left[\sum_{t\in\cTG}\langle\vlambda_t, \vrho-\vc_t(x)\rangle\right]\ge0$. This can be easily proved via the following chain of inequalities:
    \begin{align}
    \mathop{\mathbb{E}}_{x\sim\xi^*}&\left[\sum\limits_{t\in\cTG}\langle\vlambda_t, \vrho-\vc_t(x)\rangle\right]=\frac{\rho+\beta}{1+\beta}\sum\limits_{t\in \cTG}\langle\vlambda_t,\vrho-\vc_t(x^*)\rangle+\frac{1-\rho}{1+\beta}\sum\limits_{t\in \cTG}\langle\vlambda_t,\vrho-\vc_t(\nullx)\rangle\nonumber\\
    &=\sum\limits_{t\in \cTG}\langle\vlambda_t,\vrho\rangle-\frac{\rho+\beta}{1+\beta}\sum\limits_{t\in \cTG}\langle\vlambda_t,\vc_t(x^*)\rangle-\frac{1-\rho}{1+\beta}\sum\limits_{t\in \cTG}\langle\vlambda_t,\vc_t(\nullx)\rangle\nonumber\\
    &\ge \rho\sum\limits_{t\in \cTG}\|\vlambda_t\|_1-\frac{\rho+\beta}{1+\beta}\sum\limits_{t\in\cTG}\|\vlambda_t\|_1+\frac{1-\rho}{1+\beta}\beta\sum\limits_{t\in\cTG}\|\vlambda_t\|_1\nonumber\\
    &=\left(\rho-\frac{\rho+\beta}{1+\beta}+\beta\frac{1-\rho}{1+\beta}\right)\sum\limits_{t\in\cTG}\|\vlambda_t\|_1\nonumber\\
    &=0.\label{eq:th13}
    \end{align}

    Then, we can use the defintion of the regret of the primal regret minimizer to find that:
    \begin{align}
        \mathop{\mathbb{E}}_{x\sim\xi^*}\left[\sum\limits_{t\in\cTG}f_t(x)+\langle\vlambda_t, \vrho-\vc_t(x)\rangle\right]-\sum\limits_{t\in\cTG}\mleft( f_t(x_t)+\langle\vlambda_t,\vrho-\vc_t(x_t)\rangle\mright)\le \cump[\cTG],
    \end{align}
    and by rearranging and using the inequality of~\Cref{eq:th13} we have that:
    \[
    \sum\limits_{t\in\cTG} f_t(x_t)+\langle\vlambda_t,\vrho-\vc_t(x_t)\rangle\ge \mathop{\mathbb{E}}_{x\sim\xi^*}\left[\sum\limits_{t\in\cTG}f_t(x) \right]-\cump[\cTG].
    \]
   By building upon this inequality we obtain the following:
    \begin{align}
    \sum\limits_{t\in\cTG} f_t(x_t)+\langle\vlambda_t,\vrho-\vc_t(x_t)\rangle&\ge \mathop{\mathbb{E}}_{x\sim\xi^*}\left[\sum\limits_{t\in\cTG}f_t(x) \right]-\cump[\cTG]\nonumber\\
    &= \mathop{\mathbb{E}}_{x\sim\xi^*}\left[\sum\limits_{t\le T}f_t(x)-\sum\limits_{t\in \cT_\nullx}f_t(x) \right]-\cump[\cTG]\nonumber\\
    &\ge \mathop{\mathbb{E}}_{x\sim\xi^*}\left[\sum\limits_{t\le T}f_t(x)\right]-|\cT_\nullx|-\cump[\cTG]\nonumber\\
    &=\frac{\rho+\beta}{1+\beta}\sum\limits_{t\le T} f_t(x^*)+\frac{1-\rho}{1+\beta}\sum\limits_{t\le T} f_t(\nullx)-|\cT_\nullx|-\cump[\cTG]\nonumber\\
    &\ge \frac{\rho+\beta}{1+\beta}\sum\limits_{t\le T} f_t(x^*)-|\cT_\nullx|-\cump[\cTG]\nonumber\\
    &= \frac{\rho+\beta}{1+\beta}\OPT_{\vgamma}-|\cT_\nullx|-\cump[\cTG].\label{eq:th16}
    \end{align}
\xhdr{Concluding.}
    Using the ineuality of Equation~\eqref{eq:th15} and the inequality of Equation~\eqref{eq:th16} we have
    \begin{align*}
        \sum\limits_{t\le T} f_t(x_t)&\ge \sum\limits_{t\in\cTG} f_t(x_t)\\
        &\ge \sum\limits_{t\in\cTG}\langle\vlambda_t,\vc_t(x_t)-\vrho\rangle + \frac{\rho+\beta}{1+\beta}\OPT_{\vgamma}-|\cT_\nullx|-\cump[\cTG]\\
        &\ge |\cT_\nullx|-\frac{2}{\nu}-\cumd[\cTG[<\tau]] -\cumd[\cTG[>\tau]] + \frac{\rho+\beta}{1+\beta}\OPT_{\vgamma}-|\cT_\nullx|-\cump[\cTG]\\
        &=\frac{\rho+\beta}{1+\beta}\OPT_{\vgamma}-\frac{2}{\nu}-\cumd[\cTG[<\tau]] -\cumd[\cTG[>\tau]]-\cump[\cTG],
    \end{align*}
    which concludes the proof.
\end{proof}

\section{Proofs Omitted from Section~\ref{sec:stoc}}

\lemmastoc*
\begin{proof}
    We start by proving that the ineuqality of Equation~\eqref{eq:HoefEmpty2} holds with probability $1-\delta/2$.
    Let $K=|\cT_\nullx|$.
    Then, we can easily see that, for $i\in[m]$, with probability $1-\nicefrac{\delta}{(2mT)}$
    \begin{align*}
    \sum\limits_{k\in[K]}(c_{k,i}(\nullx)-\bar c_{i}(\nullx))&\le \sqrt{2K\log\left(\frac{4mT}{\delta}\right)}\le \sqrt{2T\log\left(\frac{4mT}{\delta}\right)},
    \end{align*}
    where the first inequality holds by Hoeffding's bound.
    By taking a union bound over all possible lengths of $\cT_\nullx$ (which are $T$) and $i\in [m]$, we obtain that for all possible sets $\cT_\nullx$ it holds:
    \[
    \sum\limits_{t\in\cT_\nullx}\left(c_{t,i}(\nullx)-\bar c_{i}(\nullx)\right)\le \sqrt{2T\log\left(\frac{4mT}{\delta}\right)}
    \]
    with probability at least $1-\delta/2$. Then the proof of the inequality of Equation~\eqref{eq:HoefEmpty2} is concluded by observing that:
    \[
    \sum\limits_{t\in\cT_\nullx}\bar c_i(\nullx)\le-\beta|\cT_\nullx|,
    \]
    and that $\sqrt{2T\log\left(\nicefrac{4mT}{\delta}\right)}\le M\cume$.
    Equation~\eqref{eq:HoefEmpty3} can be proved in a similar way. 
    Indeed, for any fixed $\cTG$ of size $K=|\cTG|$, and for any strategy mixture $\xi$, by Hoeffding we have that with probability at least $1-\nicefrac{\delta}{2T}$ the following holds
    \begin{align*}
    \mathop\mathbb{E}_{x\sim\xi}\left[\sum_{t \in  \cTG } f_{t}(x)+\langle\vlambda_t,\vrho-\vc_t(x)\rangle \right]&\ge \mathop\mathbb{E}_{x\sim\xi}\left[\sum_{ t \in  \cTG} \bar f(x)+\langle\vlambda_t,\vrho-\bar \vc(x)\rangle\right]-(1+2M)\sqrt{\frac{K}{2}\log\left(\frac{4T}{\delta}\right)}\\
    &\ge \mathop\mathbb{E}_{x\sim\xi}\left[\sum_{ t \in  \cTG} \bar f(x)+\langle\vlambda_t,\vrho-\bar \vc(x)\rangle\right]-(1+2M)\sqrt{\frac{T}{2}\log\left(\frac{4T}{\delta}\right)}\\
    &\ge \mathop\mathbb{E}_{x\sim\xi}\left[\sum_{ t \in  \cTG} \bar f(x)+\langle\vlambda_t,\vrho-\bar \vc(x)\rangle\right]-M\cume.
    \end{align*}
    By taking a union bound over all possible $T$ lengths of $\cTG$, we obtain that for all possible sets $\cTG$, the equation above holds with probability $1-\delta/2$.

    The Lemma follows by a union bound on the two equations above, which hold separately with probability $1-\delta/2$.
\end{proof}

\thstoc*
\begin{proof}

    The proof follows a similar approach to Theorem~\ref{thm:adv} with extra details regarding concentration inequalities to exploit stochasticity of the environment. We sketch here the proof for the sake of clarity.
    In order to simplify the notation of the proof, we will write $\cumd[\cTG[<\tau]]$ (resp., $\cumd[\cTG[>\tau]]$) in place of $\cumd[\cTG[<\tau]](\cD)$ (resp., $\cumd[\cTG[>\tau]](\cD)$).

    \xhdr{Part 1.} 
    In the same fashion as the proof of Theorem~\ref{thm:adv}, we can define $i^*$ to be the index of the resource that had budget less then $1$ at round $\tau$, and show that
    \begin{equation}\label{eq:th21}
    \sum\limits_{t\le\tau-1 }c_{t,i^*}(x_t)>B-1,
    \end{equation}
    which is exactly Equation~\eqref{eq:th11} from Theorem~\ref{thm:adv}.
    Moreover, in rounds $t\in\cT_\nullx\subset[\tau]$ we have that $x_t=\nullx$, and by Equation~\eqref{eq:HoefEmpty2} we obtain
    \begin{equation}\label{eq:th22}
        \sum\limits_{t\in\cT_\nullx} c_{t,i^*}(\nullx)\le -\beta|\cT_{\nullx}| +\cume,
    \end{equation}
    which follows Equation~\eqref{eq:th11} from Theorem~\ref{thm:adv}, with the addition of the $\cume$ term (see definition in \Cref{sec:stoc}), due to the concentration inequality. 
    Then, since $[\tau]=\cT_{\nullx}\cup\cTG[<\tau]$, by combining Equation~\eqref{eq:th21} and Equation~\eqref{eq:th22} we can conclude that:
    \begin{align}
        \sum\limits_{t\in\cTG[<\tau]} c_{t,i^*}(x_t) &= \sum\limits_{t\le\tau} c_{t,i^*}(x_t)-\sum\limits_{t\in\cT_\nullx} c_{t,i^*}(x_t)     \ge B-2+\beta|\cT_\nullx|-\cume\label{eq:th23}.
    \end{align}
    Then, through a chain of inequalities similar to the one of Equation~\eqref{eq:th14} we can conclude that
    \begin{equation}
        \sum\limits_{t\in\cTG[< \tau]} \langle\vlambda_t,\vc_t(x_t)-\vrho\rangle\ge|\cT_\nullx|-\frac{2}{\nu}-\frac{1}{\nu}\cume-\cumd[\cTG[< \tau]].
    \end{equation}
    The same arguments used for Equation~\ref{eq:th15} readily imply that
    \begin{equation}\label{eq:th24}
    \sum\limits_{t\in\cTG} \langle\vlambda_t,\vc_t(x_t)-\vrho\rangle\ge|\cT_\nullx|-\frac{2}{\nu}-\frac{1}{\nu}\cume-\cumd[\cTG[< \tau]]-\cumd[\cTG[>\tau]].
    \end{equation}

    \xhdr{Part 2.}
    Define $\xi^*$ to be the optimal stochastic policy that achieves $\OPTLP_{\bar f, \bar \vc}=\sup_{\xi\in\Xi}\mathbb{E}_{x\sim\xi}[\bar f(x)]$ and $\mathbb{E}_{x\sim\xi}[\vc(x)]\preceq \vrho$.
    By the definition of regret of $\cRp$ on the rounds $\cTG$ with respect to $\xi^*$ the following holds
    \begin{align}
        \sum\limits_{t\in\cTG} f_t(x_t)+\langle\vlambda_t,\vrho-\vc_t(x_t)\rangle&\ge \mathbb{E}_{x\sim\xi^*}\left[\sum\limits_{t\in\cTG} f_t(x)+\langle\vlambda_t,\vrho-\vc_t(x)\rangle\right]-\cump[\cTG]\nonumber\\
        &\ge \mathbb{E}_{x\sim\xi^*}\left[\sum\limits_{t\in\cTG} \bar f(x)+\langle\vlambda_t,\vrho-\bar \vc(x)\rangle\right]-\cump[\cTG]-\cume \nonumber  \\ 
        &\ge T\cdot \OPTLP_{\bar f, \bar \vc}-|\cT_\nullx| -\cump[\cTG]-\cume,\label{eq:th25}
    \end{align}
    where we used Equation~\eqref{eq:HoefEmpty2} in the first inequality, and $[T]=\cTG\cup\cT_{\nullx}$ in the third one.

    \xhdr{Concluding.} By combining Equation~\eqref{eq:th24} and Equation~\eqref{eq:th25} we can conclude that:
    \begin{align}
        \sum\limits_{t\le T} f_t(x_t)\ge T\cdot \OPTLP_{\bar f, \bar \vc} -\cump[\cTG]-\frac{2}{\nu}-\frac{1}{\nu}\cume-\cumd[\cTG[< \tau]]-\cumd[\cTG[>\tau]].
    \end{align}
    which concludes the proof.
\end{proof}

\section{Proofs Omitted from \Cref{sec:choosing}}

\lemmaboundLM*

\begin{proof}

We will address the stochastic and adversarial case separately. First, we focus on the stochastic case. The adversarial case will follow via minor modifications.

\xhdr{Stochastic setting.} We prove the theorem by contradiction. 
Suppose that there exists a round $t_2$ such that $\lVert \vlambda_{t_2}  \rVert_{1} \ge 8m/\nu$,
and let $t_1\in [t_2]$ be the the first round such that $\lVert \vlambda_{t_1} \rVert_{1}   \ge 1/\nu$. %
Notice that the dual regret minimizer $\cRd$ (\ie OGD) guarantees that:
\begin{equation*}\label{eq:mu t1 mu t2}
\lVert\vlambda_{t_1} \rVert_{1}   \le \frac{1}{\nu}+ m \eta\le \frac{2}{\nu} \quad\text{\normalfont and }\quad  \lVert \vlambda_{t_2} \rVert_{1} \le \frac{8m}{\nu}+ m\eta\le \frac{9m}{\nu},
\end{equation*} 
since the dual losses are in $[-1,1]^m$, and by assumption $\|\vlambda_{t_1-1}\|_1\le 1/\nu$ and $\|\vlambda_{t_2-1}\|_1\le 8m/\nu$.
Hence, the range of payoffs of the primal regret minimizer $|
\lossp[t]|$ in the interval $\cT:=\{t_1,\ldots, t_2\}$ can be bounded as follows
\begin{align*}
\sup\limits_{x\in \cX, t\in \cT}|\lossp[t](x)|&\le \sup\limits_{x\in \cX, t\in \cT}\Big\{|f_t(x)|+\lVert\vlambda_t\rVert_1\cdot\lVert\vrho-\vc_t(x)\rVert_\infty\Big\}\le1+2\frac{9m}{\nu} \le \frac{19m}{\nu}.
\end{align*}
Therefore, by assumption, the regret of the primal regret minimizer is at most:
\[
\cump[\cT]\le\mleft(\frac{19m}{\nu}\mright)^2 \uppp[T,\delta].
\] 

Similarly to the proof of Lemma~\ref{lm:HoefEmpty}, by applying a Hoeffding's bound to all the intervals and a union bound, we get that, with probability at least $1-\delta$, it holds
\begin{align}\label{eq:hoef}
	 \sum_{t \in \cT} \langle\vlambda_t, \vc_t(\nullx)\rangle&\le \sum_{t \in \cT} \langle\vlambda_t, \bar\vc(\nullx)\rangle + \frac{9m}{\nu}\cume\nonumber\\
  &\le -\beta\sum_{t \in \cT} \|\vlambda_t \|_1 + \frac{9m}{\nu}\cume.
\end{align}

Then, by the no-regret property of the primal regret minimizer we have
\begin{align}
    \sum_{t\in\cT} (f_t(x_t) -  \langle\vlambda_t ,c_{t}&(x_t)- \vrho\rangle)  \ge \sum_{t\in\cT} (f_t(\nullx)-  \langle\vlambda_t ,c_{t}(\nullx)- \vrho\rangle) - \mleft(\frac{19m}{\nu}\mright)^2 {\uppp[T,\delta]}\nonumber \\
    &\ge  \sum_{t\in\cT} f_t(\nullx)+\beta\sum_{t\in\cT}  \lVert \vlambda_t \rVert_1+\sum\limits_{t\in\cT}\langle\vlambda_t,\vrho\rangle-  \mleft(\frac{9m}{\nu}\mright)\cume - \mleft(\frac{19m}{\nu}\mright)^2  {\uppp[T,\delta]} \nonumber \\
    &\ge \nu\sum_{t\in\cT}  \lVert \vlambda_t \rVert_1-  \mleft(\frac{9m}{\nu}\mright)\cume - \mleft(\frac{19m}{\nu}\mright)^2  {\uppp[T,\delta]}\nonumber \\
    &\ge (t_2-t_1) -  \mleft(\frac{9m}{\nu}\mright)\cume - \mleft(\frac{19m}{\nu}\mright)^2 {\uppp[T,\delta]},\label{eq:intermediate1}
\end{align}
where in the second inequality we use~\cref{eq:hoef} and in the last one we use that $ \lVert\vlambda_t \rVert_1\ge 1/\nu$  for $t \in \cT$.

For each resource $i \in [m]$, we consider two cases: i) the dual regret minimizer never has to perform a projection operation during $\cT$, and ii) $\tilde t_i\in \cT$ is the last time in which $\lambda_{t_i,i}=0$.
In both cases, we show that
\begin{equation}\label{eq:lemmabounded_twocases}
    \sum\limits_{t\in\cT}\langle\vlambda_t, \vrho-\vc_t(x_t)\rangle\le \sum\limits_{i\in[m]}\left[\frac{\lambda_{t_1,i}-\lambda_{t_2,i}}{\eta\nu}\right]^-+ \frac{5m}{\nu^2\eta}.
\end{equation}
In the first case, since we are using OGD as the dual regret minimizer and, by assumption, it never has to perform projections during $\cT$, it holds that for all resources $i\in[m]$:
\[
\lambda_{t_2,i} = \eta  \sum_{t \in \cT} (c_{t,i}(x_{t})- \rho)+\lambda_{t_1,i}.
\]
Now consider the Lagrange multiplier $\vlambda^*$ such that, for each $i\in[m]$,
\begin{equation}\lambda^*_i:=
    \begin{cases}
        \nicefrac{1}{\nu} &\text{if}\quad\sum_{t \in \cT} (c_{t,i}(x_{t})- \rho)\ge0\\
        0 &\text{otherwise}
    \end{cases}.
\end{equation}
By exploiting \Cref{lem:OGD} for a single component $i\in[m]$, we have that:
\begin{align}
   \sum\limits_{t\in\cT}(\lambda_{i}^*-\lambda_{t,i})(c_{t,i}(x_t)-\rho)&\le \frac{(\lambda^*_i-\lambda_{t_1,i})^2}{2\eta}+\frac12\eta T,
\end{align}
which yields the following
\begin{align*} 
    \sum_{t \in \cT}\lambda_{t,i}\cdot(\rho-c_{t,i}(x_t)) & \le 	\sum_{t \in \cT}\lambda^*_i\cdot (\rho-c_{t,i}(x_t))+  \frac{(\lambda^*_i-\lambda_{t_1,i})^2}{2\eta}+\frac12\eta T\\
    &\le\frac{1}{\nu}\left[\frac{\lambda_{t_1,i}-\lambda_{t_2,i}}{\eta}\right]^{-}+\frac{(\lambda^*_i-\lambda_{t_1,i})^2}{2\eta}+\frac12\eta T.\numberthis\label{eq:intermediate2}
\end{align*}
Then, since $\|\vlambda_{t_1-1}\|_1\le 1/\nu$ by construction, it holds $\lambda_{t_1-1,i}\le1/\nu $. Hence, since the dual regret minimizer is OGD and its utilities are in $[-1,1]^m$, it holds 
\begin{align}\label{eq:smallLambda}
    \lambda_{t_1,i}\le\frac1\nu+\eta.
\end{align}
Then, we have
\[
\frac{(\lambda^*_i-\lambda_{t_1,i})^2}{2\eta} \le \frac{\left( \max\{\lambda^*_i,\lambda_{t_1,i} \}\right)^2}{2\eta}\le \frac{1}{2\eta} \left( \frac{1}{\nu}+\eta \right)^2.
\]
It is easy to verify that for all $\nu$, $T$, and $\eta\le\frac{1}{2\sqrt{T}}$:\footnote{
Notice that the definition of $\eta$ satisfies $\eta\le\frac{1}{2\sqrt{T}}$.}
\[
\frac{1}{2\eta} \left( \frac{1}{\nu}+\eta \right)^2+\frac{1}{2}\eta T \le \frac{5}{\nu^2\eta},
\]
which implies that
\[
\sum_{t \in \cT} \lambda_{t,i},\rho-c_{t,i}(x_t)\le \frac{1}{\nu}\left[\frac{\lambda_{t_1,i}-\lambda_{t_2,i}}{\eta}\right]^{-}+\frac{5}{\nu^2\eta}.
\]

In the second case we define $\tilde t_i$ as the last time step $t\in[t_1,t_2]$ in which $\lambda_t=0$. Thus, the following holds:
\begin{equation}\label{eq:lembounded0}
\lambda_{t_2,i} = \eta  \sum_{t \in [\tilde t_i,t_2]} (c_{t,i}(x_{t})- \rho).
\end{equation}
Now, consider the Lagrange multipliers $\lambda_{i,1}=0$ and $\lambda_{i,2}=\frac{1}{\nu}$. 
By~\Cref{lem:OGD}, we have that the regret of the dual regret minimizer with respect to $0$ over $[t_1, \tilde t_i]$ is bounded by:
\begin{align}\label{eq:lembounded1}
   \sum\limits_{t\in[t_1, \tilde t_i-1]}\lambda_{t,i}(\rho-c_{t,i}(x_t))&\le \frac{\left(\lambda_{1,i}-\lambda^*_{i,1} \right)^2}{2\eta}+\frac12\eta T\le \frac{\left(\frac{1}{\nu}+\eta\right)^2}{2\eta}+\frac12\eta T,
\end{align}
where in the last inequality we use \cref{eq:smallLambda}.

Similarly, on the interval $[\tilde t_i+1, t_2]$ we have that the regret of the dual regret minimizer with respect to $1/\nu$ is bounded by
\begin{align*}
   \sum\limits_{t\in[\tilde t_i, t_2]} \left(\lambda_{i,2}-\lambda_{t,i}\right)\cdot(c_{t,i}(x_t)-\rho) \le \frac{\left( \lambda^*_{i,2}- \lambda_{\tilde t_{i,i}}\right)^2}{2\eta} + \frac12\eta T \le  \frac{1}{2\eta}\frac{1}{\nu^2}+\frac12\eta T,
\end{align*}
which can be rearranged into
\begin{align*}
   \sum\limits_{t\in[\tilde t_i, t_2]} \lambda_{t,i}\cdot(\rho-c_{t,i}(x_t)) &\le \frac1\nu\sum\limits_{t\in[\tilde t_{i}, t_2]}(\rho-c_{t,i}(x_t))+\frac{1}{2\eta}\frac{1}{\nu^2}+\frac12\eta T\\
   &=-\frac{\lambda_{t_2,i}}{\eta\nu}+\frac{1}{2\eta}\frac{1}{\nu^2}+\frac12\eta T \\
   &\le\left[\frac{\lambda_{t_1,i}-\lambda_{t_2,i}}{\eta\nu}\right]^-+\frac{1}{2\eta}\frac{1}{\nu^2}+\frac12\eta T\numberthis\label{eq:lembounded2}
\end{align*}
where  the equality follows from  \Cref{eq:lembounded0} and the last inequality from  $-x\le\min(y-x,0)$ for $x,y\ge0$.

Then by summing \Cref{eq:lembounded1} and \Cref{eq:lembounded2} we obtain
\begin{align*} 
    \sum_{t \in [t_1,t_2]}\lambda_{t,i}\cdot(\rho-c_{t,i}(x_t)) & \le \left[\frac{\lambda_{t_1,i}-\lambda_{t_2,i}}{\eta\nu}\right]^-+\frac{\left(\frac{1}{\nu}+\eta\right)^2}{2\eta}+\frac{1}{2\eta\nu^2}+\eta T\numberthis\label{eq:intermediate2}.
\end{align*}

If we take $\eta\le\frac1{2\sqrt{T}}$, then the following inequality holds for all $\nu$ and $T$:\footnote{
Notice that the definition of $\eta$ satisfies $\eta\le\frac{1}{2\sqrt{T}}$.}
\[
\frac{\left(\frac{1}{\nu}+\eta\right)^2}{2\eta}+\frac{1}{2\eta\nu^2}+\eta T \le \frac{5}{\nu^2\eta}.
\]

This concludes the second case and concludes the proof of  \Cref{eq:lemmabounded_twocases}.

Now we can sum over all resources $i\in[m]$ and conclude that:
\begin{align*}
\sum\limits_{t\in\cT}\langle \vlambda_t, \vrho-\vc_t(x_t)\rangle&\le \sum\limits_{i\in[m]}\left[\frac{\lambda_{t_1,i}-\lambda_{t_2,i}}{\eta\nu}\right]^-+\frac{5m}{\eta\nu^2}\\
&\le \frac{\|\vlambda_{t_1}\|_1-\|\vlambda_{t_2}\|_1}{\eta\nu}+\frac{5m}{\eta\nu^2}\\
&\le-\frac{6m}{\nu^2\eta}+\frac{5m}{\eta\nu^2}\\
&=-\frac{m}{\nu^2\eta}.
\end{align*}

Finally, the cumulative utility of the primal regret minimizer over $\cT$ is bounded by
\begin{align*} 
    \sum_{t \in [t_1,t_2]}\mleft(f_t(x_t)+ \langle \lambda_t,\vrho-c_t(x_t)\rangle \mright) &\le (t_2-t_1)-\frac{m}{\nu^2\eta}  .\numberthis\label{eq:intermediate2}
\end{align*}

By putting \cref{eq:intermediate1} and \cref{eq:intermediate2} together we have that
\begin{align*}
(t_2-t_1)-  \frac{m}{\nu^2\eta} & \ge   (t_2-t_1) -  \mleft(\frac{9m}{\nu}\mright)\cume - \mleft(\frac{19m}{\nu}\mright)^2 {\uppp[T,\delta]}\\
&\ge   (t_2-t_1) -  \mleft(\frac{9m}{\nu}\mright)\cume - \mleft(\frac{19m}{\nu}\mright)^2 {\uppp[T,\delta]}.
\end{align*}
Hence, 
\begin{align} \label{eq:contr1}
\frac{m}{\nu^2 \eta} \le  \mleft(\frac{9m}{\nu}\mright)\cume + \mleft(\frac{19m}{\nu}\mright)^2 {\uppp[T,\delta]} .
\end{align}
which holds by the assumption that there exists a time $t_2$ such that $\|\vlambda_{t_2}\|_1\ge \frac{8m}{\nu}$.

Thus, if we set 
\[
\eta \defeq \left(18 \cume +361m \uppp[T,\delta]+2m\sqrt{T}\right)^{-1},
\]
we reach a contradiction with Equation~\eqref{eq:contr1} by observing that
\[
\frac{m}{\nu^2\eta}= \frac{m}{\nu^2}\left(18 \cume+ 361m \uppp[T,\delta] +m\nu\sqrt{T}\right)>\left(\frac{9m}{\nu}\right)\cume + \mleft(\frac{19m}{\nu}\mright)^2 {\uppp[T,\delta]}.
\]
This concludes the proof for the stochastic setting.

\xhdr{Adversarial setting.}
In the adversarial setting all the passages above still apply by setting $\cume=0$. This is because, using \Cref{ass:adv}, we can refrain from using the concentration inequality originating the term $\cume$. 

This concludes the proof.
\end{proof}

\corollaryFINAL*

\begin{proof}
    Lemma~\ref{lem:bounded_lagrangian} allows us to bound the $\ell_1$-norm of the Lagrange multipliers which are played by $\|\vlambda_t\|_1\le8m/\nu$.
    This fact, by using Lemma~\ref{lem:OGD}, readily gives a bound on the terms $\cumd[\cTG<\tau](\cD)$ and $\cumd[\cTG>\tau](\cD)$ by observing that 
    \[
    \sup\limits_{\substack{\vlambda, \vlambda',\\ \|\vlambda\|_1, \|\vlambda'\|_1\le\frac{8m}{\nu}, }}\|\vlambda-\vlambda'\|_2^2\le \frac{64m^3}{\nu^2}.
    \]
    This, together with the definition of $\eta=(k_1 \cume +k_2 m \uppp[T,\delta]+2m\sqrt{T})^{-1}$ yields the bound 
    \begin{align*}
    \max(\cumd[\cTG<\tau],\cumd[\cTG>\tau])&\le \frac{32m^3}{\nu^2}(k_1 \cume +k_2 m \uppp[T,\delta]+2m\sqrt{T})+\frac{m}{2}\frac{T}{k_1 \cume +k_2 m \uppp[T,\delta]+2m\sqrt{T}}\\
    &=k_1\frac{32m^3}{\nu^2}\cume + k_2 \frac{32m^4}{\nu^2}\uppp[T,\delta] + \frac{32m^4}{\nu} \sqrt{T} + \frac{1}{4}\sqrt{T}\\
    &\le \frac{k_3}{2} \frac{m^4}{\nu^2}\left(\cume + \uppp[T,\delta]+\sqrt{T}\right).
    \end{align*}

    Then, by Assumption~\ref{ass:primal_no_beta} and Lemma~\ref{lem:bounded_lagrangian}, the regret of the primal regret minimizer $\cRp$ is bounded by:
    \begin{align}
    \cump[\cTG]&\le \left(1+\frac{16m}{\nu}\right)^2\uppp\\
    &=k_4\frac{m^2}{\nu^2}\uppp.
    \end{align}

    This concludes the proof by leveraging Theorem~\ref{thm:adv} and Theorem~\ref{th:stoc}.
\end{proof}

\newpage

\end{document}